\documentclass[letterpaper, 10pt, conference]{ieeeconf} 

\IEEEoverridecommandlockouts                              
\usepackage{xcolor}

\usepackage{graphics} 
\usepackage{epsfig} 
\usepackage{mathptmx} 
\usepackage{times} 
\usepackage{amsmath} 

\DeclareMathOperator*{\argmin}{arg\,min}
\usepackage{amssymb}  
\usepackage{amsbsy}
\usepackage{comment}
\usepackage{relsize}
\usepackage{graphicx}      
\usepackage{balance}
\usepackage{lipsum}
\usepackage{ifpdf}
\usepackage[mathcal]{euscript}
\usepackage{algorithm}

\usepackage{algorithmic}
\usepackage{enumerate}
\usepackage{amsfonts,amssymb}
\usepackage{mathtools}
\usepackage{soul, color}

\newtheorem{thm}{Theorem}
\newtheorem{lemma}{Lemma}
\newtheorem{assumption}{Assumption}

\usepackage{graphics} 
\usepackage{epsfig,epstopdf} 
\usepackage{mathptmx} 
\usepackage{times} 
\usepackage{amsmath} 
\usepackage{amssymb}  
\usepackage{amsbsy}
\usepackage{comment}

\usepackage[multiple]{footmisc}
\usepackage{cite}
\usepackage{graphicx}
\usepackage{longtable}
\graphicspath{{figures/}} 
\usepackage{multirow}
\usepackage{algorithm}
\usepackage{algorithmic}
\floatname{algorithm}{Algorithm 1:}

\usepackage{tikz}
\usetikzlibrary{arrows.meta,shapes,positioning}
\usetikzlibrary{shapes,arrows}
\tikzstyle{block} = [draw, rectangle, minimum height=2em, minimum width=4em]
\tikzstyle{sum} = [draw, circle, node distance=1.5cm]
\tikzstyle{input} = [coordinate]
\tikzstyle{output} = [coordinate]
\tikzstyle{feedback} = [coordinate]
\usetikzlibrary{arrows}
\usetikzlibrary{automata}
\usetikzlibrary{arrows,chains}
\usetikzlibrary {positioning}
\usetikzlibrary{graphs, graphs.standard}
\usetikzlibrary{fit,calc,shapes.geometric}
\tikzset{edge/.style = {->,> = latex}}
\usepackage{subfigure}
\usepackage[10pt]{moresize}

\title{\LARGE \bf
Boosting Fairness and Robustness in Over-the-Air Federated Learning*

    \author{Halil Yigit Oksuz$^{1,2}$, Fabio Molinari$^{1}$, Henning Sprekeler$^{2,3}$, J\"{o}rg Raisch$^{1,2} $ }

    \thanks{* This work has been funded by the Deutsche Forschungsgemeinschaft (DFG, German Research Foundation) under Germany’s Excellence Strategy – EXC 2002/1 “Science of Intelligence” – project number 390523135.}
    
    \thanks{$^{1}$ H.Y. Oksuz, F. Molinari, and J. Raisch are with the Control Systems Group at Technische  Universität Berlin, Germany. {\tt\small \{oksuz@tu-berlin.de\}, \{molinari,raisch\}@control.tu-berlin.de}. } 
    
    \thanks{$^{2}$ H.Y. Oksuz, H. Sprekeler, and J. Raisch are with Exzellenzcluster Science of Intelligence, Technische Universität Berlin, Marchstr. 23, 10587, Berlin, Germany.}

    \thanks{$^{3}$ H. Sprekeler is with the Modelling Cognitive Processes Group at Technische  Universität Berlin, Germany.
}
}

\begin{document}
\setlength{\footnotesep}{0.1cm}

\maketitle
\thispagestyle{empty}
\pagestyle{empty}

\begin{abstract}                

Over-the-Air Computation is a beyond-5G communication strategy that has recently been shown to be useful for the decentralized training of machine learning models due to its efficiency. In this paper, we  propose an Over-the-Air federated learning algorithm that aims to provide fairness and robustness  through minmax optimization. By using the epigraph form of the problem at hand, we show that the proposed algorithm converges to the optimal solution of the minmax problem. Moreover, the proposed approach does not require reconstructing channel coefficients by complex encoding-decoding schemes as opposed to state-of-the-art approaches. This improves both efficiency and privacy.  
\end{abstract}



\vspace{2mm}
\section{INTRODUCTION}
\vspace{1mm}

In a traditional federated learning setting (as in \cite{mcmahan2017communication, smith2017federated,bonawitz2019towards, yang2019federated}, and some references therein), we consider a system of $N$ agents connected to a central unit, and their objective is to accomplish a machine learning task in a decentralized manner. In a supervised learning setting, each agent $i$ has its own local dataset represented by $D_i = \{d^{n}_i \}_{n=1}^{|D_i|}$, where $|D_i|$ is the number of data points and $i=1,2,\cdots,N$. The dataset $D_i$ consists of pairs of inputs $u^{n}_i$ and labels $z^{n}_i$, i.e., $d^{n}_i = (u^{n}_i , z^{n}_i)$, and the objective is to build a global parametric model that is able to predict the correct labels of the given data points. To this end, each agent $i$ uses the private local cost function
\begin{equation}\label{eq:local_cost}
g_i(\theta)= \frac{1}{|D_i|}\sum_{n=1}^{|D_i|} \mathcal{L}_i(d^{n}_i,\theta),
\end{equation} 
\noindent
where $\mathcal{L}_i(d^{n}_i,\theta)$ is the error function representing the difference between the predicted output of the model with parameter $\theta$ and the actual label of the given data point $d^{n}_i$. If we define the average of all local cost functions as the global cost function $g(\theta)$, the objective of the overall system is to solve the constrained optimization problem
\begin{equation}\label{eq:global_from_local}  
\theta^{*} = \argmin_{\theta\in \Theta }  g(\theta)=\argmin_{\theta\in \Theta }\frac{1}{N} \sum_{i=1}^{N} g_i(\theta),
\end{equation} 
\noindent 
where $\Theta  \subset \mathbb{R}^m$ is a nonempty constraint set. If the central unit had access to the datasets of all agents, a centralized gradient descent-based optimization could be employed to address the global learning task at hand \cite{nedic2020distributed}. However, in a federated learning setting, e.g., \cite{bonawitz2019towards,chellapandi2023convergence, omori2023combinatorial}, each agent has access only to its own local (possibly private) dataset. Having carried out the local optimization steps, they transmit the updated versions of the local parameter estimates to the central unit. Subsequently, the central unit aggregates these local parameter estimates and transmits the aggregated version to the agents for the next optimization steps.

For large-scale systems, where information exchange and cooperation are vital, one critical challenge for the averaging-based federated learning algorithms is heterogeneity \cite{sery2021over,gafni2022federated,ye2023heterogeneous}. When the data is heterogeneous, i.e., the fact that agents observe data from different distributions, the parameter vectors minimizing the local cost functions will in general vary significantly between different agents, and minimizing the cost function (\ref{eq:global_from_local}) may not be desirable. Instead, using a worst-case optimization problem may reflect practical requirements more accurately. Another problem that we encounter in large-scale networks is that the communication load on the overall system increases
with the number of agents \cite{konevcny2016federated,li2020federated,kairouz2021advances,oksuz2023federated}. When multiple agents transmit information at the same time and in the same frequency band, signals are affected by the physical phenomenon of interference. Standard communication protocols\footnote{In TDMA (Time Division Multiple Access), agents are assigned different time slots when they can transmit, whereas in FDMA (Frequency Division Multiple Access), different frequency bands are allocated to different users.} prevent interference by transmitting signals orthogonally. However, these techniques are not resource-efficient in the sense that they increase the need for bandwidth or the number of communication rounds, which in general leads to a decrease in total throughput and efficiency \cite{molinari2021max,frey2021over}.

In this paper, we present a federated learning algorithm that aims to improve the performance of the worst-performing agent in the system, thus providing fairness and robustness against heterogeneity. We leverage a beyond-5G communication strategy, called Over-the-Air Computation, which is more efficient as the number of  agents grows \cite{frey2021over}. Unlike the existing literature on Over-the-Air computation, e.g, \cite{yang2020federated, sery2021over}, the proposed algorithm can operate despite the inherently unknown nature of channel coefficients. We do not assume any knowledge of  (nor the capability to reconstruct) channel coefficients, and therefore we will not need extra pre-processing efforts to reconstruct the channel,  which makes the proposed scheme more time and resource-efficient. Moreover, since the channel coefficients are completely unknown, privacy is inherently guaranteed, as discussed in \cite{molinari2020exploiting}.

The remainder of this paper is organized as follows: we present the problem setup in Section II. In Section III, we introduce the proposed algorithm, whose convergence analysis is presented in Section IV. A numerical example is presented in Section V. Concluding remarks are given in Section VI.

\vspace{1mm}
\section{PROBLEM SETUP}
\vspace{1mm}

The set of real numbers is denoted by $\mathbb{R}$, and $\mathbb{R}^{m}$ represents $m$-dimensional Euclidean space. $\mathbb{N}$ and $\mathbb{N}_0$ respectively denote the set of natural numbers and the set of nonnegative integers. Given a finite set $T$, its cardinality is denoted by $|T|$. For a vector $x\in \mathbb{R}^{m}$, $x^{T}$ denotes its transpose.  Euclidean norm of the vector $x\in \mathbb{R}^{m}$ is denoted by $||x||$. The projection of  $x\in \mathbb{R}^{m}$ onto a nonempty closed convex set $S\subset \mathbb{R}^{m}$ is denoted by $\mathbf{P}_{S}(x)$, i.e., $\mathbf{P}_{S}(x) = \argmin_{s\in S}|| s-x ||$. Projection is non-expansive, i.e., $||\mathbf{P}_\Theta(x)-\mathbf{P}_\Theta(y) ||\leq ||x-y||$ holds for all $x,y \in \mathbb{R}^m$ if $\Theta$ is a nonempty closed convex set  (see \cite{bertsekas2003convex}). For given $a, b\in \mathbb{R}$, the function $\max\{a,b\}$ takes value $a$ if $a>b$, and $b$ otherwise. The expected value of a random variable $p$ is denoted by $\mathbb{E}[p]$. Given a logical argument $g_i(x)$, $\mathbf{1}_{\{g_i(x)\} }$ denotes the function that takes value 1 when $g_i(x)$ is true and 0 when $g_i(x)$ is false. For a function $f: \mathbb{R}^{m}\to (-\infty, \infty \mathbf{]}$, we define $D_f=\{x\in \mathbb{R}^{m}\big| f(x)<\infty \}$. Then, for a subgradient of a convex function $f$ with respect to $x$ at $\tilde{x}\in D_f$, denoted by $\partial_{x} f(\tilde{x})$, the inequality $f(\mathbf{\tilde{x}}) + \partial f^{T}(\mathbf{\tilde{x}})(\mathbf{x}-\mathbf{\tilde{x}}) \leq f(\mathbf{x})$ holds for all $\mathbf{x}\in D_f$.

\vspace{0.5mm}
\subsection{Minmax Reformulation}
\vspace{0.5mm}

In a federated learning setting with $N$ agents, where $V=\{1,2,\cdots, N \}$ denotes the index set, we are interested in improving the performance of the worst-performing agent by solving the following optimization problem:
\begin{align}
\label{eq:minmax_orj}
     \min_{\theta\in \Theta}\text{} \max_{i\in V} g_i(\theta),
\end{align}
\noindent
We aim to compute a parameter vector estimate minimizing the worst-case loss observed among all agents, thus providing some form of fairness \cite{mohri2019agnostic,hu2022federated}. However, it is difficult and inefficient to use (\ref{eq:minmax_orj}) directly for federated learning purposes. Instead, we can consider an alternative (epigraph) form:
\begin{alignat}{2}
\label{eq:minmax_recast}
    \nonumber
    &\min_{\alpha \in \mathbb{R}}\quad &&\alpha \\
    &\text{subject to}  \quad &&g_i(\theta)\leq \alpha \qquad  \forall \theta\in \Theta,\; i\in V;
\end{alignat}
\noindent
where the optimal value $\alpha^*$ is assumed to be finite. Moreover, 
$g_i(\theta^*)\leq \alpha^*$ holds if $\theta^*\in \Theta$ is optimal for (\ref{eq:minmax_recast}). Let $\bar{g}_i\big((\theta,\alpha)\big) = \max\{ g_i(\theta)-\alpha, 0 \}$ and $p_i > 0$ respectively denote a penalty for violating the constraint in (\ref{eq:minmax_recast}) for $i\in V$  and a weight of this penalty. Then, by following \cite{bertsekas1975necessary} and \cite{srivastava2011distributed}, one can rewrite (\ref{eq:minmax_recast})  as 
\begin{alignat}{2}
\label{eq:minmax_recast_penalty}
    &\min_{\theta\in \Theta, \alpha \in \mathbb{R}}\quad \alpha + \sum_{i=1}^{N} p_i \bar{g}_i\big((\theta,\alpha)\big). 
\end{alignat}

It has been shown in \cite{bertsekas1975necessary} that any solution for (\ref{eq:minmax_recast_penalty}) is also a solution of (\ref{eq:minmax_recast})  if $p_i > 1$ for all $i\in V$. Hence, (\ref{eq:minmax_recast_penalty})  can be considered as the global loss function for the federated learning setting.

\subsection{Over-the-Air Communication Model}\label{sec:wmac}
\vspace{0.5mm}

In wireless communication systems, the wireless multiple access channel (WMAC) model has been extensively used to characterize communication between multiple transmitters 
and a single receiver over fading channels, e.g.,  \cite{ahlswede1973multi, Giridhar2006}.
Throughout this paper, we employ the WMAC-based communication model described
by~\cite{molinari2021max,molinari2022over}. Let each agent $i\in V$  simultaneously transmit information $\mathbf{s}_i(k)\in\mathbb{R}^{m}$ to the central unit in the same frequency band at each time step $k \in \mathbb{N}_0$. However, this information is corrupted  by the channel and superimposed by the receiver, i.e., the received information by the central unit is given by 
\begin{equation}\label{eq:interference}
	\mathbf{s}^{rec}(k) =\sum_{i=1}^{N}\lambda_{i}(k)\mathbf{s}_{i}(k),
\end{equation}
\noindent
where the $\lambda_{i}(k)$ are unknown time-varying positive channel coefficients, i.e., $\lambda_{i}(k)>0$ for all $i=1,2,\cdots, N$.

Note that employing Over-the-Air computation has two main advantages:
$(i)$ the channel coefficients $\lambda_{i}(k)$ in (\ref{eq:interference})
are unknown, and it is impossible to reconstruct $s_i(k)$ from $\mathbf{s}^{rec}(k)$, which inherently provides \textit{privacy}; $(ii)$ our approach does not require any knowledge of the channel coefficients, nor do they need to be reconstructed. This makes our algorithm highly efficient, in particular for large $N$. This is demonstrated in numerical experiments in Section \ref{sec:sim}.

\vspace{1mm}
\section{FEDERATED FAIR OVER-THE-AIR LEARNING (FedFAir) ALGORITHM}
\label{sec:TheRes}
\vspace{1mm}

\begin{algorithm}[tb]
\caption{FedFAir}
\begin{algorithmic}[1]
	\REQUIRE  $\theta(0)\in\Theta$, $\alpha(0) \in \mathbb{R}$, and $p_i \in \mathbb{R}$\vspace{0.5mm}
	\ENSURE
	\FOR {each time step $k\in \mathbb{N}_0$}
        \STATE \text{The central unit computes:}
        \begin{align}
            \label{eq:server_1}
            v(k) = \alpha(k) -  \frac{\eta(k)}{N}
        \end{align}
	\STATE \text{The central unit broadcasts $\theta(k)$ and $v(k)$} 
	\STATE \text{Each agent $i$ updates its local variables:}
	\begin{align}
           \label{eq:local_agent_update_theta}
    &\theta_i(k) = \theta(k) -  \eta(k)p_i\partial_{\theta} \bar{g}_i\big((\theta(k),v(k))\big)\\ 
		\label{eq:local_agent_update_alpha}
  		&\alpha_i(k) = v(k) -  \eta(k)p_i\partial_{\alpha} \bar{g}_i\big((\theta(k),v(k))\big)\\ 
		&\varrho_i(k)=1
	\end{align}\vspace{-5mm}
	\STATE \text{Each agent $i$ transmits $\theta_i(k)$, $\alpha_i(k)$, and $\varrho_i(k)$}
	\STATE \text{The central unit receives:}
	\begin{align}
		\label{eq:ser_rec_1}
		\theta^{\text{rec}}(k)&=\sum_{i=1}^{N}\lambda_{i}(k)\theta_i(k) \\  
		\label{eq:ser_rec_3}
		\alpha^{\text{rec}}(k)&=\sum_{i=1}^{N}\lambda_{i}(k)\alpha_i(k)\\  
		\label{eq:ser_rec_2}
		\varrho^{\text{rec}}(k)&=\sum_{i=1}^{N}\lambda_{i}(k)\varrho_i(k)
	\end{align}
	\STATE \text{The central unit updates:}
	\begin{align}\label{eq:server_comp}
		\theta(k+1) &= \mathbf{P}_\Theta\Bigg(\frac{\theta^{\text{rec}}(k)}{\varrho^{\text{rec}}(k)}\Bigg) \\ 
        \label{eq:server_comp_2}
        \alpha(k+1) &= \frac{\alpha^{\text{rec}}(k)}{\varrho^{\text{rec}}(k)}
	\end{align}
	\ENDFOR
\end{algorithmic}
\label{alg:oafl}
\end{algorithm}
\setlength{\textfloatsep}{10pt}

The FedFAir algorithm is summarized in Algorithm 1. At the beginning, the central unit selects $\theta(0)\in \Theta$ and $\alpha(0)\in  \mathbb{R}$. Then, through iterations, the central unit computes (\ref{eq:server_1}) and broadcasts $v(k)$ and $\theta(k)$. Subsequently, each agent computes its own  $\alpha_i(k)$ and $\theta_i(k)$ by using the local update rules  (\ref{eq:local_agent_update_theta}) and (\ref{eq:local_agent_update_alpha}) with the step size $\eta (k)$, respectively.  Afterward, in the first round of agent-to-central unit communications, all agents transmit simultaneously (and in the same frequency band) their local values $\alpha_i(k)$ to the central unit. In the second round, each agent transmits simultaneously (and in the same frequency band) their local parameter vectors $\theta_i(k)$. Finally, in the third round, the constant $\varrho_i(k)=1$ is transmitted by all agents again simultaneously and in the same frequency band. We assume that the delays between the three rounds are sufficiently small such that the channel coefficients can be considered constant over the three rounds. According to the WMAC model, the central unit receives the vector (\ref{eq:ser_rec_1}) and two scalars given in (\ref{eq:ser_rec_3}) and (\ref{eq:ser_rec_2}) at each time step $k\in \mathbb{N}_0$. Finally, the central unit computes (\ref{eq:server_comp}) and (\ref{eq:server_comp_2}), which can be rewritten in the following form: 
\begin{align}\label{eq:developFEDOTA}
\theta(k+1) &=  \mathbf{P}_\Theta\Bigg(\sum_{i=1}^{N}h_{i}(k)	\theta_i(k)\Bigg), \\ \label{eq:developFEDOTA_2}
\alpha(k+1) &= \sum_{i=1}^{N}h_{i}(k) \alpha_i(k),
\end{align}
\noindent
where $\mathbf{P}_\Theta$ is the projection operator onto the set $\Theta$,  $h_{i}(k)=\frac{\lambda_{i}(k)}{ \sum_{i=1}^{N}\lambda_{i}(k)}$ are the normalized channel coefficients, and the $\lambda_{i}(k)$ are the unknown time-varying positive real channel coefficients. By construction, the $h_{i}(k)$ are also positive, and, for all $k\geq 0$,

\begin{equation}\label{eq:eq_asmp_h}
\sum_{i=1}^{N} h_{i}(k)=1.
\end{equation}

Next, we state our assumptions on individual objective functions and the step size as follows:

\begin{assumption}
\label{ass:conset}
The constraint set $\Theta\subset\mathbb{R}^{m}$ is convex and compact. As a consequence (see {\cite[Theorem~2.41, p.40]{rudin1976principles}}),   $\Theta$ is then also closed and bounded.
\end{assumption}

\begin{assumption}
	\label{ass:g_i}
	The individual objective functions $g_i(\theta)$ are convex over $\mathbb{R}^m$. Consequently, the minimax problem (\ref{eq:minmax_orj}) is also convex since the point-wise maximum of convex functions preserves convexity {\cite[Proposition~1.2.4]{bertsekas2003convex}}.  In addition, the local cost functions $\bar{g}_i\big((\theta,\alpha)\big)$ are also convex but nondifferentiable,  and their sets of subdifferentials are nonempty for all $x\in \mathbb{R}$, and $i\in V$ {\cite[Proposition~4.2.1]{bertsekas2003convex}}.
\end{assumption}

\begin{assumption}
	\label{ass:step_size} 
	The step size $\eta(k)$ in the FedFAir algorithm is chosen to satisfy $\eta(k)>0$, $\sum_{k=0}^{\infty}\eta(k) = \infty$, and $\sum_{k=0}^{\infty}\eta^{2}(k) < \infty$.
\end{assumption}

\begin{assumption}
	\label{ass:alphas_1}
    The unknown time-varying positive real channel coefficients $\lambda_i(k)$ ($i=1,2,\cdots, N$) are assumed to be independent across time and agents.
\end{assumption}

We refer here to {\cite[Ch~2.3, Ch~2.4]{Tse2012}} and {\cite[Ch~5.4]{molisch2012wireless}}, thus considering channel coefficients independent realizations (see \cite{sklar1997rayleigh}).

\section{CONVERGENCE PROPERTIES}

We start with the following lemmas that are essential for the proofs presented in the paper.

\begin{lemma}[{\cite[Thm. 3.4.2]{bartle2000introduction}}]\label{lem:subsLim}
    If a sequence  $\{a(k)\}$ of real numbers  converges to a real number $x$, then any subsequence $\{a(k_t)\}$  of $\{a(k)\}$ also converges to $x$.
\end{lemma}

\begin{lemma}[{\cite[Lem. 11, p.50]{polyak1987introduction}}]
\label{lem:nedic}
Let $\{v(k)\}$,  $\{b(k)\}$,  $\{u(k)\}$, and  $\{c(k)\}$ be nonnegative sequences of random variables. Suppose that 
\begin{itemize}
    \item[$(i)$] $\sum_{k=0}^{\infty} b(k)< \infty$ and $\sum_{k=0}^{\infty} c(k)< \infty$ hold almost surely,
    \item[$(ii)$] For each $k\in \mathbb{N}_0$, the following holds almost surely 
    \begin{align}\label{eq_lem_nedic}
        \mathbb{E}\big[v(k+1)|F_k\big]\leq \big( 1+b(k)\big)v(k)-u(k)+c(k),
    \end{align}
    \noindent
    where $\mathbb{E}\big[v(k+1)|F_k\big]$ denotes the conditional expectation for the given $F_k = \{v(t), u(t), c(t), t=0,1, \cdots,k\}$.
\end{itemize}
 Then, $\{v(k)\}$ converges to some $v\geq0$ and $\sum_{k=0}^{\infty} u(k)< \infty$ almost surely.
\end{lemma}

\begin{lemma}[{\cite[Prop. 2]{alber1998projected}}]
\label{lemma:limitUpforinfty2}
Let $\{a(k)\}$ be a nonnegative sequence and $\{b(k)\}$ an eventually nonnegative sequence, i.e., $\exists\tilde{k}\geq 0$ such that $b(k)\geq 0$  $\forall k\geq \tilde{k}$. Let $\sum_{k=0}^{\infty} a(k) =  \infty$ and $\sum_{k=0}^{\infty} a(k)b(k) < \infty$. Then, there exists a subsequence $ \{b (k_t) \}$ of $\{b (k) \}$ such that $\lim_{t\to\infty} b (k_t) = 0$.
\end{lemma}

\begin{lemma}\cite{srivastava2011distributed}
\label{lem:srivastava2011distributed}
    Let $\alpha^* = \min_{\theta\in \Theta}\text{} \max_{i\in V} g_i(\theta)$, $\theta^* = \argmin_{\theta\in \Theta}\max_{i\in V} g_i(\theta)$,  and $p_i >1$ for all $i\in V$. Then, for $\bar{g}_i\big((\theta, \alpha)\big) = \max\{g_i(\theta) -\alpha,0\}$, we have $\sum_{i=1}^{N}p_i\bar{g}_i\big((\theta, \alpha)\big) + \alpha \geq \alpha^* $ for all $\theta\in \Theta$ and $\alpha \in \mathbb{R}$, where equality holds if and only if $\theta = \theta^*$ and $\alpha = \alpha^*$. 
\end{lemma}

We are now ready to present the main result.

\begin{thm}
\label{thm:thm_main_try}
Suppose that Assumptions \ref{ass:conset}, \ref{ass:g_i}, \ref{ass:step_size}, and \ref{ass:alphas_1}  hold. Let $\theta^* \in \Theta$ and $\alpha^* \in \mathbb{R}$ respectively be an optimal solution and the optimal value for problem (\ref{eq:minmax_orj}). If $p_i >\max\Big(1, \frac{1}{N\mathbb{E}[h_i(k)]}\Big)$ for all $i\in V$, then $\lim_{k\to\infty} \theta(k)=\theta^{*}$ and $\lim_{k\to\infty} \alpha(k)=\alpha^{*}$ with probability 1.
\end{thm}
\begin{proof} 
We start by introducing   
$\beta(k)=\begin{bmatrix} \theta (k)^T \;\; \alpha(k) \end{bmatrix}^T$, $y(k)=\begin{bmatrix} \theta (k)^T \;\; v(k) \end{bmatrix}^T$. Then, for any $\beta^{*}=\begin{bmatrix} \theta^*{}^T \;\; \alpha^*
\end{bmatrix}^T\in\Theta\times \mathbb{R}$, by using  (\ref{eq:server_1})-(\ref{eq:developFEDOTA_2}), the non-expansive property of the projection $\mathbf{P}_{\Theta}$, and  the fact that $\theta^* = \mathbf{P}_{\Theta}(\theta^*)$,  we have 
\begin{align}
		\label{eq:pre_1}
		\nonumber
		&||\beta(k+1)-\beta^{*}||^2 =\Bigg|\Bigg|\begin{bmatrix} \mathbf{P}_{\Theta}\Big(\sum_{i=1}^{N}h_{i}(k)	\theta_i(k)\Big)-\mathbf{P}_{\Theta}(\theta^*) \\ 
    \sum_{i=1}^{N}h_{i}(k) \alpha_i(k) - \alpha^* \end{bmatrix}\Bigg|\Bigg|^2 \\ \nonumber
    & \leq  
    \Bigg|\Bigg|\begin{bmatrix}
        \theta(k)-\theta^* - \eta(k)\sum_{i=1}^{N}h_i(k)p_i\partial_{\theta} \bar{g}_i\big((\theta(k),v(k))\big)
        \\
        v(k) - \alpha^* - \eta(k)\sum_{i=1}^{N}h_i(k)p_i \partial_{\alpha} \bar{g}_i\big((\theta(k),v(k))\big)
    \end{bmatrix} \Bigg|\Bigg|^2 \\ 
    &= \Big|\Big|
    y(k)-\beta^{*}-\eta(k)\sum_{i=1}^{N}h_i(k)p_i\partial\bar{g}_i\big(y(k)\big)\Big)\Big|\Big|^2,
	\end{align}
\noindent 
where $\partial\bar{g}_i\big(y(k)\big)=  \begin{bmatrix}
    \partial_{\theta} g_i\big(\theta(k)\big)  &-1
\end{bmatrix}^T \mathbf{1}_{\{g_i(\theta)\geq \alpha\} }.$
We can further expand (\ref{eq:pre_1}) as
\begin{align}
		\label{eq:try_eq_1}
		\nonumber
		||\beta(k+1)-\beta^{*}||^2 &\leq \bigg|\bigg|  y(k)-\eta(k)\sum_{i=1}^{N}h_i(k)p_i\partial\bar{g}_i\big(y(k)\big)-\beta^{*} \bigg|\bigg|^2 \\  \nonumber
		&=\big| \big|  y(k)- \beta^{*} \big| \big|^{2}\\ \nonumber 
  & \quad -  2\eta(k)  \sum_{i=1}^{N}h_i(k)p_i \partial\bar{g}^T_i\big(y(k)\big)\big(y(k)- \beta^{*}\big) \\  
		& \qquad \qquad +  \Big| \Big|  \eta(k) \sum_{i=1}^{N}h_i(k)p_i\partial\bar{g}_i\big(y(k)\big)\Big| \Big|^{2}. 
	\end{align}

For the first term on the right-hand side of (\ref{eq:try_eq_1}), we have 
\begin{align}
    \label{eq:y_to_beta}
    \nonumber
    \big| \big|  y(k)- \beta^{*} \big| \big|^{2} &= \Bigg| \Bigg|  \begin{bmatrix}
        \theta(k)-\theta^* \\ \alpha(k)-\alpha^* 
    \end{bmatrix} - \begin{bmatrix}
        0 \\ \frac{\eta(k)}{N}
    \end{bmatrix}  \Bigg| \Bigg|^{2} \\ 
    &= \big| \big|  \beta(k)- \beta^{*} \big| \big|^{2} -\frac{2\eta(k)}{N}\big(\alpha(k)- \alpha^{*}\big) + \frac{\eta^2(k)}{N^2}.
\end{align}
\noindent
which follows from (\ref{eq:server_1}). Note that the boundedness of the constraint set $\Theta$  implies  $\exists  L_\Theta>0$ such that $\big|\big| \partial\bar{g}_i\big(y(k)\big) \big|\big| \leq L_\Theta$, which together with (\ref{eq:eq_asmp_h}), the convexity of the function $|| \cdot ||^{2}$, and the fact that $1<p_i\leq p_{max}$ with  $p_{max} = \max_{i} p_i$ ($i\in V$) can be used to find an  upper-bound for the last term on the right-hand side of (\ref{eq:try_eq_1}) as 
\begin{align}
		\label{eq:try_eq_2} 
		\nonumber
		\Big| \Big|  \eta(k)  \sum_{i=1}^{N}h_i(k)p_i\partial\bar{g}_i\big(y(k)\big)\Big| \Big|^{2}
		&= \eta^{2}(k) \Big| \Big|    \sum_{i=1}^{N}h_i(k)p_i\partial\bar{g}_i\big(y(k)\big)\Big| \Big|^{2}\\ \nonumber
		&\leq \eta^{2}(k)\sum_{i=1}^{N}h_i(k)p_i^{2} \Big| \Big|    \partial\bar{g}_i\big(y(k)\big)\Big| \Big|^{2}\\
		&\leq  \eta^{2}(k)M_1.
\end{align} 
\noindent
where $M_1 = p^2_{max}L_\Theta^{2}$. Let $F_k$ represent the past iterates of $\alpha (k)$ and $\theta (k)$, i.e., $F_k = \{\alpha (t),\: \theta (t), t=0, 1, \cdots, k \}$ for $k\in\mathbb{N}_0$. Subsequently, by using (\ref{eq:y_to_beta}), (\ref{eq:try_eq_2}),  and taking the expectation conditioned on $F_k$ of both sides of (\ref{eq:try_eq_1}) together with the linearity of the expectation, we obtain
\begin{align}
	\label{eq:try_eq_12}
	\nonumber
	\mathbb{E}\big[||\beta(k+1)&-\beta^{*}||^2 | F_k\big] \leq\big| \big|  \beta(k)- \beta^{*} \big| \big|^{2} -\frac{2\eta(k)}{N}(\alpha(k)- \alpha^{*})
	\\ \nonumber 
	&-2\eta(k)\mathbb{E}\big[  \sum_{i=1}^{N}h_i(k)p_i \partial\bar{g}^T_i\big(y(k)\big)\big(y(k)- \beta^{*}\big)| F_k\big] \\  
	& \qquad \qquad \qquad +  \eta^{2}(k)M_2,
	\end{align}
        \noindent
        where $M_2= M_1  + \frac{1}{N^2}$. Note that due to Assumption \ref{ass:alphas_1}, the statistics of $h_i(k)$ ($i=1,2,\cdots, N$) are independent of $h_i(t)$ for $t<k$, which also implies that  $y (k)$ and  $h_i(k)$ are statistically independent at time $k$ since  the statistics of $y (k)$ are dependent only on $h_i(t)$ for $t<k$ and $i=1,2,\cdots, N$. Hence, $\mathbb{E}\big[h_i(k)| F_k\big] = \mathbb{E}\big[h_i(k)\big]$ holds for all $k\in\mathbb{N}_0$. Then, by using  the linearity of the expectation again, we can  write the third term on the right-hand side of (\ref{eq:try_eq_12}) as 
	\begin{align}
		\label{eq:try_eq_123}
		\nonumber
    &-2\eta(k)\mathbb{E}\big[ \sum_{i=1}^{N}h_i(k)p_i \partial\bar{g}^T_i\big(y(k)\big)\big(y(k)- \beta^{*}\big)| F_k\big] \\ \nonumber
    &=-2\eta(k)  \sum_{i=1}^{N}\mathbb{E}\big[h_i(k)p_i\partial\bar{g}^T_i\big(y(k)\big)\big(y(k)- \beta^{*}\big)| F_k\big] \\ 
		&=  -2\eta(k)  \sum_{i=1}^{N}\mathbb{E}\big[h_i(k)\big] p_i\partial\bar{g}^T_i\big(y(k)\big)\big(y(k)- \beta^{*}\big). 
	\end{align}
        \noindent
        Moreover, by Assumption \ref{ass:g_i} (convexity of $\bar{g}_i(\cdot)$), we have 
	\begin{align}
	\label{eq:try_eq_3}
        \nonumber
	 \sum_{i=1}^{N}p_i\partial_{y}  \bar{g}^{T}_i\big(y(k)\big)\big(y(k)-\beta^{*}\big) &\geq \sum_{i=1}^{N} p_i\Big(\bar{g}_i\big(y(k)\big) - \bar{g}_i\big(\beta^{*}\big)\Big)\\
        &=\sum_{i=1}^{N}p_i \bar{g}_i\big(y(k)\big),
	\end{align}
    \noindent
    which follows from the fact that for any  $\beta^{*}=\begin{bmatrix} \theta^*{}^T \;\; \alpha^* \end{bmatrix}^T\in\Theta\times \mathbb{R}$, we have $g_i(\theta^*)-\alpha^*\leq 0$, and therefore $\bar{g}_i(v^*) = \max\{g_i(\theta^*)-\alpha^*,0 \} = 0$  for all $i\in V$ and  $k\in \mathbb{N}_0$.   Thus, by using (\ref{eq:try_eq_3}), (\ref{eq:try_eq_123}) can be written as
	\begin{align}
	\label{eq:try_eq_4}
        \nonumber
-2\eta(k)  \sum_{i=1}^{N}\mathbb{E}\big[h_i(k)\big] p_i\partial\bar{g}_i\big(y(k)\big)\big(y(k)- \beta^{*}\big) \\
      \leq   -2\eta(k)  \sum_{i=1}^{N}\mathbb{E}\big[h_i(k)\big]  p_i\bar{g}_i\big(y(k)\big), 
    \end{align}
    \noindent
    which can then be used to rewrite (\ref{eq:try_eq_12}) as
\begin{align}
	\label{eq:try_eq_5}
	\nonumber
	\mathbb{E}\big[||\beta(k+1)-\beta^{*}||^2| F_k\big] &\leq\big| \big|  \beta(k)- \beta^{*} \big| \big|^{2} +  \eta^{2}(k) M_2
	\\  
	&-\frac{2\eta(k)}{N}  \Big( \alpha(k)-\alpha^* +  N\sum_{i=1}^{N} \bar{p}_i\bar{g}_i\big(y(k)\big)  \Big) 
\end{align}
\noindent
where $\bar{p}_i= \mathbb{E}\big[h_i(k)\big]p_i$. We can add and subtract the term $2\eta(k)  \big(  \sum_{i=1}^{N} \bar{p}_i\bar{g}_i\big(\beta(k)\big)  \big)$ from the right hand side of (\ref{eq:try_eq_5}) to get
\begin{align}
	\label{eq:try_eq_5_2}
	\nonumber
	\mathbb{E}\big[||\beta(k+1)-\beta^{*}||^2| F_k\big] &\leq\big| \big|  \beta(k)- \beta^{*} \big| \big|^{2}+  \eta^{2}(k) M_2
	\\ \nonumber 
	&-\frac{2\eta(k)}{N}  \Big( \alpha(k)-\alpha^* +  N\sum_{i=1}^{N} \bar{p}_i\bar{g}_i\big(\beta(k)\big)  \Big)
	\\  
	&+2\eta(k) \bar{p}_{max}  \sum_{i=1}^{N} \big| \bar{g}_i\big(\beta(k)\big)-\bar{g}_i\big(y(k)\big) \big|, 
\end{align}
\noindent
which follows from the fact that $$  \sum_{i=1}^{N} \bar{p}_i\bar{g}_i\big(\beta(k)\big)   -\sum_{i=1}^{N} \bar{p}_i\bar{g}_i\big(y(k)\big)\leq  \bar{p}_{max}\sum_{i=1}^{N} \big| \bar{g}_i\big(\beta(k)\big)-\bar{g}_i\big(y(k)\big) \big|  $$
where $\bar{p}_{max} = \max_{i} \bar{p}_i$ ($i\in V$). By using (\ref{eq:server_1}) and the relation $|\text{max}\{x_1,0\}-\text{max}\{x_2,0\} | \leq | x_1 - x_2 |$ for scalars $x_1$ and $x_2$, one can obtain 
\begin{align}
\label{eq:abs_func_diff}
\nonumber
    |\bar{g}_i\big(\beta(k)\big) - \bar{g}_i\big(y(k)\big)|&\leq |g_i\big(\theta(k)\big) - \alpha(k) - g_i\big(\theta(k)\big) + v(k)  |
    \\
    &= \frac{\eta (k)}{N},  
\end{align}
\noindent
which can then be used to obtain 
\begin{align}
	\label{eq:try_eq_5_3}
	\nonumber
	\mathbb{E}\big[||\beta(k+1)-\beta^{*}||^2| F_k\big] &\leq\big| \big|  \beta(k)- \beta^{*} \big| \big|^{2}
	\\ \nonumber 
	&-\frac{2\eta(k)}{N}  \Big( \alpha(k)-\alpha^* +  \sum_{i=1}^{N} w_i\bar{g}_i\big(\beta(k)\big)  \Big)\\  
	& \qquad \qquad \qquad +  \eta^{2}(k) M_3. \\[-3mm] \nonumber
\end{align}
\noindent
where $w_i = N\bar{p}_i$ and $M_3= M_2+2p_{max}$. Suppose that $p_i>\max\Big(1, \frac{1}{N\mathbb{E}[h_i(k)]}\Big)$. Then,
\begin{align}
    \label{eq:p_2}
w_i =N\bar{p}_i = N\mathbb{E}\big[h_i(k)\big]p_i>1.
\end{align}

Note that since $\alpha^* = \min_{\theta\in \Theta}\text{} \max_{i\in V} g_i(\theta)$, by Lemma \ref{lem:srivastava2011distributed}, we have
\begin{align}
    \label{eq:costlemm3}
    \alpha(k)-\alpha^* +  \sum_{i=1}^{N} w_i\bar{g}_i\big(\beta(k)\big)  \geq 0
\end{align}
\noindent
for all $k\in \mathbb{N}_0$. Moreover, by Assumption \ref{ass:step_size}, we have $\sum_{k=0}^{\infty}  \eta^{2}(k) < \infty $, which together with (\ref{eq:try_eq_5_3}) and Lemma \ref{lem:nedic} give
\begin{align}
    \label{eq:costlemm2}
    &\lim_{k\to \infty} \big| \big|  \beta(k)- \beta^{*} \big| \big|^{2} = \vartheta\geq 0, \\
    \label{eq:thrtalemm2}
    &\sum_{k=0}^{\infty}   \eta(k)\Big(\alpha(k)-\alpha^* +  \sum_{i=1}^{N} w_i\bar{g}_i\big(\beta(k)\big)  \Big)  < \infty,
\end{align}
\noindent
for any $\beta^{*}=\begin{bmatrix} \theta^*{}^T \;\; \alpha^* \end{bmatrix}^T\in\Theta\times \mathbb{R}$ and $\alpha^*\in \mathbb{R}$ with probability 1. Additionally,  $\sum_{k=0}^{\infty}  \eta(k) = \infty $ also holds by Assumption \ref{ass:step_size}, which together with Lemma \ref{lemma:limitUpforinfty2} imply that there exists a subsequence such that the following holds with probability 1,
\begin{align}
\label{eq:liminf_0}
    \lim_{l\to\infty} \Big(\alpha(k_l)-\alpha^* +  \sum_{i=1}^{N} w_i\bar{g}_i\big(\beta(k_l)\big)\Big)   = 0,
\end{align}
\noindent
which also implies that along a further subsequence, we have limit points $\lim_{l\to\infty} \alpha(k_l) =\bar{\alpha}$ and $\lim_{l\to\infty} \beta(k_l) =\bar{\beta}$ with probability 1, which satisfy
\begin{align}
    \label{eq:costlemm5}
    \bar{\alpha}-\alpha^* +  \sum_{i=1}^{N} w_i\bar{g}_i\big(\bar{\beta}\big)  = 0.
\end{align}
\noindent
Hence, by Lemma \ref{lem:srivastava2011distributed} and the fact that $w_i>1$ for all $i\in V$, it follows that $\bar{\alpha}= \alpha^*$ and $\bar{\theta}= \theta^*$ are  optimal  for the minmax problem (\ref{eq:minmax_orj}) with probability 1. This implies that $\lim_{l\to\infty} \big| \big|  \beta(k_l)- \beta^{*} \big| \big|^{2} = 0$ with probability 1.
Hence, by  Lemma \ref{lem:subsLim}, we have $\vartheta = 0$ in (\ref{eq:costlemm2}) with probability 1 and that concludes the proof.
\end{proof}

\vspace{2mm}
\section{NUMERICAL EXAMPLE}
\label{sec:sim}
\vspace{1mm}

\begin{figure}[tb]
	\centering
	\includegraphics[scale=0.58]{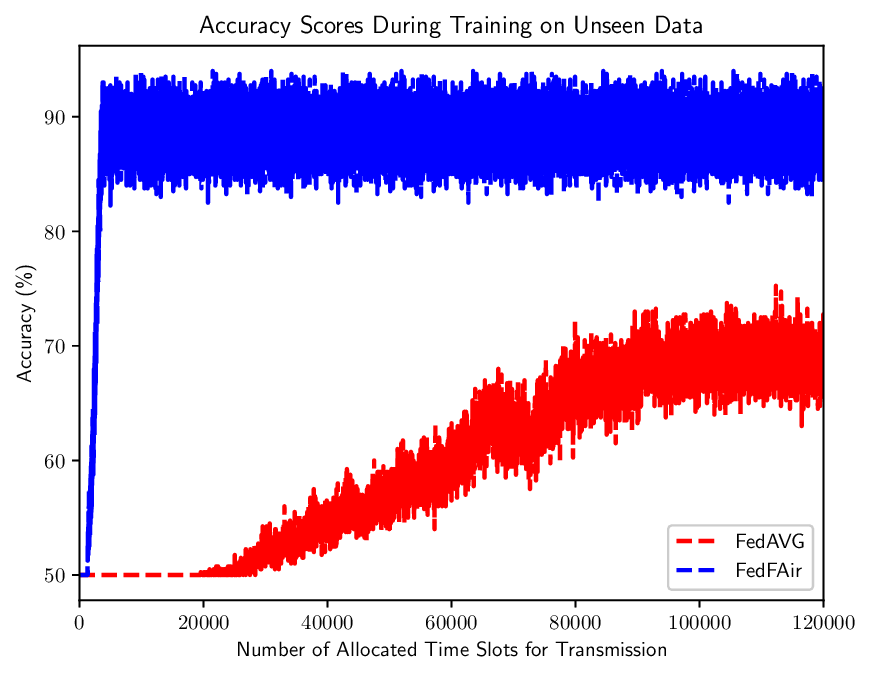}
	\caption{Comparison of the performances  FedFAir and FedAVG algorithms on unseen data during training in terms of their accuracy.}
	\label{fig:param}
\end{figure}

We consider a federated logistic regression problem, where each agent has its own private data. The objective is to collaboratively build a global model that can carry out a binary classification task. In this setting, each agent has 2 different classes of data, labeled by 0 or 1. We represent the dataset of the $i$-th agent by $D_i = \{d^{n}_i \}_{n=1}^{ |D_i|}$, where $d^{n}_i = (u^{n}_i , z^{n}_i) \in \mathbb{R}^{m}\times \{0, 1\}$. The objective is to find a separation rule so that agents can identify the correct classes of some unseen data from different classes. To this end, the following convex loss function  is used by each agent:
\begin{align}\label{eq:crossEnt}
	\nonumber
	g_i(\theta,d_i) =-\frac{1}{|D_i|}\bigg(&\sum_{n=1}^{|D_i|}  z^{n}_i \log\big(S(\theta^{T}u^{n}_i)\Big)
	\\[-1mm]   
	&+  (1-z^{n}_i) \log\big(1-S(\theta^{T}u^{n}_i)\Big)\bigg) 
\end{align}
\noindent
where  $S(x)=\frac{1}{1+e^{-x}}$ and $z^{n}_i$ is the local estimate  of  the $i$-th agent. The step size is chosen as $\eta (k)= \frac{0.1}{(k+1)^{0.6}}$. 

To demonstrate the fairness and robustness of the FedFAir algorithm, we consider a system of $12$ agents, each with  a different number of training samples and with noise injected into the data available to each agent. Note that the noise injected into the datasets of different agents is sampled from different distributions. The parameter vector is of  dimension $4$, i.e.,  $\theta\in \mathbb{R}^{4}$, and $p_i>1$ is chosen for all agents $i=1,2,\cdots, 12$. Additionally, the constraint set is given by $\Theta = \{\theta\in \mathbb{R}^{4}| \hspace{1mm} ||\theta ||\leq 10 \}$.  In order to monitor the performance of the system, we randomly sample some previously unseen test data at every iteration, and then we measure performance of the global model on this test data by computing its classification accuracy during training. As it can be seen from Fig.  \ref{fig:param}, the FedFAir algorithm achieves around $90\%$ accuracy after approximately 5000 iterations. 

Notice that the FedAVG algorithm achieves approximately $70-75\%$ of classification accuracy after 100000 iterations as it can also be seen in Fig.  \ref{fig:param}. However, due to the heterogeneity in distribution of the data, we observe that the FedAVG algorithm cannot properly identify the samples of data with labels 0. This can be seen  from the confusion matrices in Fig. \ref{fig:conf_matrices}, where the FedAVG model cannot successfully recognize the data labeled as 0, whereas the FedFAir algorithm performs much better by recognizing a large percentage of data, both labelled 0 and 1.

Additionally, we compare the FedFAir algorithm and the standard FedAVG \cite{mcmahan2017communication} algorithm  in terms of communication efficiency. For the latter, the TDMA scheme is used for communication between agents and the central unit. In this case, individual time slots are allocated for each agent to transmit its parameter vector at each communication round. After receiving parameter vectors, the central unit computes the average of them and sends it back to the agents.  Since we consider a system with $N=12$ agents, it takes $N=12$ time slots per communication round for each agent to transmit its updated parameter vector while only $3$ are needed for the FedFAir algorithm (see (\ref{eq:ser_rec_2})), independent of the value of $N$. This makes the execution of one iteration of the FedFAir algorithm $N/3 = 4$ times faster than the execution of an iteration of the FedAVG algorithm.

\begin{figure}[tb]
\advance\leftskip -2mm
\subfigure[Confusion Matrix for FedAVG]{
	\includegraphics[scale=0.4275]{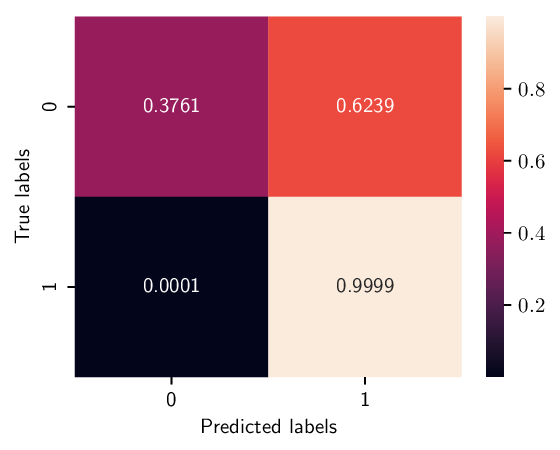}
\label{subfig:conf_matrix1}
    }
\subfigure[Confusion Matrix for FedFAir]{
	\includegraphics[scale=0.425]{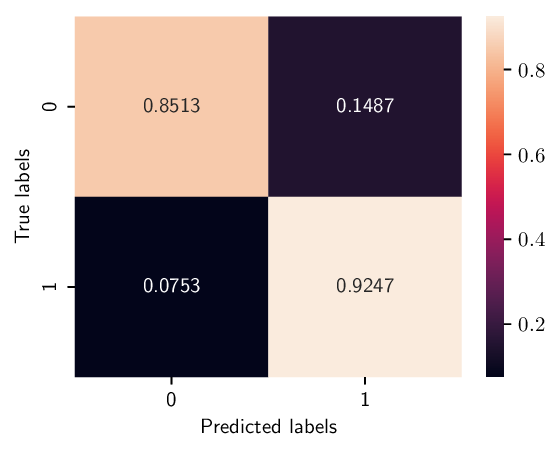}

        \label{subfig:conf_matrix2}
    }

    \caption{Comparison of Confusion Matrices}
    \label{fig:conf_matrices}
\end{figure}

\vspace{2mm}
\section{CONCLUSION}
\vspace{2mm}

In this paper, we have introduced the FedFAir algorithm which uses Over-the-Air Computation to carry out efficient decentralized learning while providing fairness and improved performance. We have shown that the FedFAir algorithm converges to an optimal solution of the minimax problem. Furthermore, we have also illustrated our theoretical findings with a numerical example.

Future research will include the development of resilient federated learning algorithms when there are malicious agents in the system.

\vspace{3mm}
\balance
\bibliography{main.bib}
\bibliographystyle{IEEEtran} 
	
\end{document}